\documentclass{article}

\usepackage{amsmath}
\usepackage{amsthm}
\usepackage{amsfonts}
\usepackage[dvipsnames]{xcolor}
\usepackage{graphicx}

\newtheorem{proposition}{Proposition}
\newtheorem{lemma}{Lemma}

\title{The Size of a Hyperball in a Conceptual Space}
\author{Lucas Bechberger\thanks{ORCID: 0000-0002-1962-1777}\\Institute of Cognitive Science, Osnabr\"uck University\\ \emph{lucas.bechberger@uni-osnabrueck.de}}
\date{}

\begin{document}
\maketitle

\begin{abstract}
The cognitive framework of conceptual spaces \cite{Gardenfors2000} provides geometric means for representing knowledge. A conceptual space is a high-dimensional space whose dimensions are partitioned into so-called domains. Within each domain, the Euclidean metric is used to compute distances. Distances in the overall space are computed by applying the Manhattan metric to the intra-domain distances. Instances are represented as points in this space and concepts are represented by regions. In this paper, we derive a formula for the size of a hyperball under the combined metric of a conceptual space. One can think of such a hyperball as the set of all points having a certain minimal similarity to the hyperball's center.
\end{abstract}

\section{Conceptual Spaces}
This section presents the cognitive framework of conceptual spaces as defined by G\"ardenfors \cite{Gardenfors2000} and introduces our formalization of dimensions, domains, and distances as described in \cite{Bechberger2017KI}.

A conceptual space is a high-dimensional space spanned by a set $D$ of so-called ``quality dimensions''. Each of these dimensions $d \in D$ represents a way in which two stimuli can be judged to be similar or different. Examples for quality dimensions include temperature, weight, time, pitch, and hue. We denote the distance between two points $x$ and $y$ with respect to a dimension $d$ as $|x_d - y_d|$.

A domain $\delta \subseteq D$ is a set of dimensions that inherently belong together. Different perceptual modalities (like color, shape, or taste) are represented by different domains. The color domain for instance consists of the three dimensions hue, saturation, and brightness.

G\"{a}rdenfors argues based on psychological evidence that distance within a domain $\delta$ should be measured by the weighted Euclidean metric: 
$$d_E^{\delta}(x,y, W_{\delta}) = \sqrt{\sum_{d \in \delta} w_{d} \cdot | x_{d} - y_{d} |^2}$$
The parameter $W_{\delta}$ contains positive weights $w_{d}$ for all dimensions $d \in \delta$ representing their relative importance. We assume that $\textstyle\sum_{d \in \delta} w_{d} = 1$.\\


The overall conceptual space $CS$ is defined as the product space of all dimensions. Again, based on psychological evidence, G\"{a}rdenfors argues that distance within the overall conceptual space should be measured by the weighted Manhattan metric $d_M$ of the intra-domain distances. Let $\Delta$ be the set of all domains in $CS$. We define the distance within a conceptual space as follows:
$$
d_C^{\Delta}(x,y,W) = \sum_{\delta \in \Delta} w_{\delta} \cdot d_E^{\delta}(x,y,W_{\delta})
= \sum_{\delta \in \Delta}w_{\delta} \cdot \sqrt{\sum_{d \in \delta} w_{d} \cdot |x_{d} - y_{d}|^2}
$$
The parameter $W\hspace{-0.2cm} = \hspace{-0.2cm}\langle W_{\Delta},\{W_{\delta}\}_{\delta \in \Delta}\rangle$ contains $W_{\Delta}$, the set of positive domain weights $w_{\delta}$. We require that $\textstyle\sum_{\delta \in \Delta} w_{\delta} = |\Delta|$. Moreover, $W$ contains for each domain $\delta \in \Delta$ a set $W_{\delta}$ of dimension weights as defined above. The weights in $W$ are not globally constant, but depend on the current context. One can easily show that $d_C^{\Delta}(x,y,W)$ with a given $W$ is a metric.\\

The similarity of two points in a conceptual space is inversely related to their distance. G\"{a}rdenfors expresses this as follows :
$$Sim(x,y) = e^{-c \cdot d(x,y)}\quad \text{with a constant}\; c >0 \; \text{and a given metric}\; d$$

Properties (like red, round, and sweet) and concepts (like apple, dog, and chair) can be represented by regions in this space: Properties are defined within a single domain and concepts are defined on the overall space. In \cite{Bechberger2017KI}, we have developed a mathematical formalization of concepts and properties.


\section{Hyperballs under the Unweighted Metric}
\label{Hyperballs}

In general, a hyperball of radius $r$ around a point $p$ can be defined as the set of all points with a distance of at most $r$ to $p$:
$$H =  \{x \in CS \;|\;d(x,p) \leq r \}$$

\begin{figure}[tp]
\centering
\includegraphics[width=\textwidth]{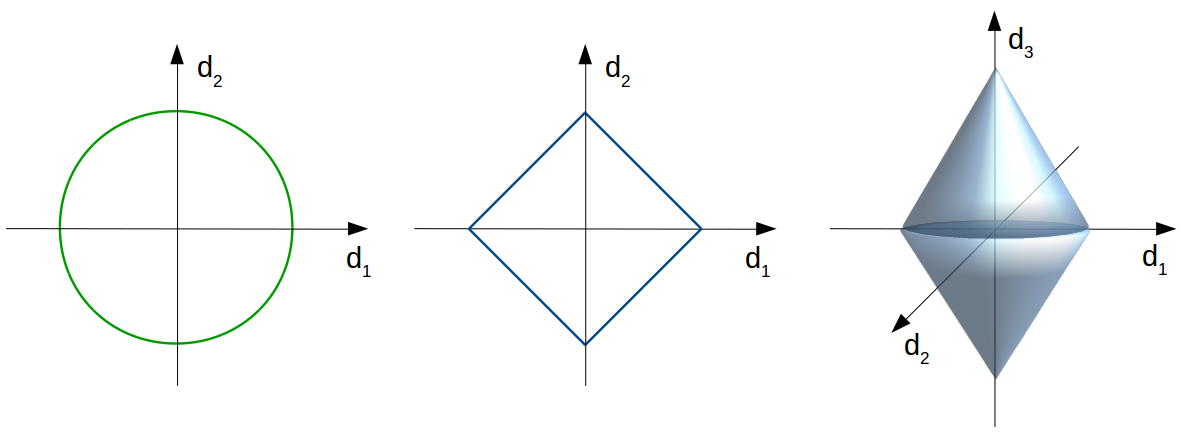}
\caption{Left: Two-dimensional hyperball under the Euclidean metric. Middle: Two-dimensional hyperball under the Manhattan metric. Right: Three-dimensional hyperball under the combined metric (with domain structure $\Delta = \{ \delta_1 = \{d_1, d_2\}, \delta_2 = \{d_3\}\}$).}
\label{fig:hyperballs}
\end{figure}

If the Euclidean distance $d_E$ is used, this corresponds to our intuitive notion of a ball -- a round shape centered at $p$. However, under the Manhattan distance $d_M$, hyperballs have the shape of diamonds. Under the combined distance $d_C^\Delta$, a hyperball in three dimensions has the shape of a double cone (cf. Figure \ref{fig:hyperballs}).

As similarity is inversely related to distance, one can interpret a hyperball in a conceptual space as the set of all points that have a minimal similarity $\alpha$ to the central point $p$, where $\alpha$ depends on the radius of the hyperball.

In this section, we assume an unweighted version of $d_C^\Delta$:
$$d_C^\Delta(x,y) = \sum_{\delta \in \Delta} \sqrt{\sum_{d \in \delta} |x_{d} - y_{d}|^2}$$

In order to derive a formula for the hypervolume of a hyperball under $d_C^\Delta$, we need to use the following three lemmata:

\begin{lemma}
\label{lemma:AngleIntegral}
$$\int_0^{2\pi} \int_0^\pi \int_0^\pi \dots \int_0^\pi \sin^{n-2}(\phi_1) \sin^{n-3}(\phi_2) \dots \sin(\phi_{n-2}) d\phi_1 \dots d\phi_{n-1} = 2 \cdot \frac{\pi^{\frac{n}{2}}}{\Gamma\left(\frac{n}{2}\right)}$$
where $\Gamma(\cdot)$ is Euler's Gamma function and $n \in \mathbb{N}$.
\end{lemma}
\begin{proof}
\begin{align*}
I &=\int_0^{2\pi} \int_0^\pi \int_0^\pi \dots \int_0^\pi \sin^{n-2}(\phi_1) \sin^{n-3}(\phi_2) \dots \sin(\phi_{n-2}) d\phi_1 \dots d\phi_{n-1}\\
&= \left(\int\displaylimits_{0}^{2\pi} 1 \;d\phi_{n - 1}\right) 
\left(\int\displaylimits_{0}^\pi \sin(\phi_{n-2}) \;d\phi_{n - 2}\right)
\cdots 
\left(\int\displaylimits_{0}^\pi \sin^{n-2}(\phi_{1}) \;d\phi_{1}\right)\\
&= \left(4 \cdot \int\displaylimits_{0}^{\frac{\pi}{2}} 1 \;d\phi_{n - 1}\right) 
\left(2 \cdot \int\displaylimits_{0}^{\frac{\pi}{2}} \sin(\phi_{n-2}) \;d\phi_{n-2}\right)
\cdots 
\left(2 \cdot \int\displaylimits_{0}^{\frac{\pi}{2}} \sin^{n-2}(\phi_{1}) \;d\phi_{1}\right)
\end{align*}
We can now use the definition of the Beta function, which is $$B(x,y) = 2 \cdot \int\displaylimits_{0}^{\frac{\pi}{2}} \sin^{2x-1}(\phi) \cos^{2y-1}(\phi)\; d\phi$$
Using $y = \frac{1}{2}$, we get:
\begin{align*}
I &= \left(4 \cdot \int\displaylimits_{0}^{\frac{\pi}{2}} 1 \;d\phi_{n - 1}\right) 
\left(2 \cdot \int\displaylimits_{0}^{\frac{\pi}{2}} \sin(\phi_{n-2}) \;d\phi_{n-2}\right)
\cdots 
\left(2 \cdot \int\displaylimits_{0}^{\frac{\pi}{2}} \sin^{n-2}(\phi_{1}) \;d\phi_{1}\right)\\
&= 2 \cdot B(\frac{1}{2},\frac{1}{2}) \cdot B(1,\frac{1}{2}) \cdots B(\frac{n-2}{2},\frac{1}{2}) \cdot B(\frac{n-1}{2},\frac{1}{2})
\end{align*}
Next, we use the identity $B(x,y) = \frac{\Gamma(x)\Gamma(y)}{\Gamma(x+y)}$ with Euler's Gamma function $\Gamma$ and the fact that $\Gamma(\frac{1}{2}) = \sqrt{\pi}$. We get:
$$I = 2 \cdot \frac{\Gamma(\frac{1}{2})\Gamma(\frac{1}{2})}{\Gamma(1)} \cdot \frac{\Gamma(1)\Gamma(\frac{1}{2})}{\Gamma(\frac{3}{2})} \cdots \frac{\Gamma(\frac{n-2}{2})\Gamma(\frac{1}{2})}{\Gamma(\frac{n-1}{2})} \cdot \frac{\Gamma(\frac{n-1}{2})\Gamma(\frac{1}{2})}{\Gamma(\frac{n}{2})} 
= 2 \cdot \frac{\Gamma(\frac{1}{2})^{n}}{\Gamma(\frac{n}{2})} 
= 2 \cdot \frac{\pi^{\frac{n}{2}}}{\Gamma(\frac{n}{2})}$$
\end{proof}

\begin{lemma}
\label{lemma:IntegralBetaFunction}
For any natural number $j > 0$ and any $a,b \in \mathbb{R}$, the following equation holds:
$$\int_0^{r-\sum_{i=1}^{j-1}r_i} r_j^{a-1} \cdot \left(r- \sum_{i=1}^j r_i\right)^b dr_j = B(a,b+1) \cdot \left(r- \sum_{i=1}^{j-1} r_i\right)^{a+b}$$
\end{lemma}
\begin{proof}
We can make a variable change by defining $r_j = \left(r- \sum_{i=1}^{j-1} r_i\right)\cdot z$ which gives $dr_j = \left(r- \sum_{i=1}^{j-1} r_i\right)\cdot dz$. This gives us:
\begin{align*}
&\int_0^{r-\sum_{i=1}^{j-1}r_i} r_j^{a-1} \cdot \left(r- \sum_{i=1}^j r_i\right)^b dr_j\\
&=\int_0^{1} \left(r- \sum_{i=1}^{j-1} r_i\right)^{a-1} \cdot z^{a-1} \cdot \left(r- \sum_{i=1}^{j-1} r_i - \left(r- \sum_{i=1}^{j-1} r_i\right)\cdot z\right)^b \cdot\left(r- \sum_{i=1}^{j-1} r_i\right) dz\\
&= \left(r- \sum_{i=1}^{j-1} r_i\right)^{a-1+b+1} \int_0^{1} z^{a-1} (1-z)^b dz = \left(r- \sum_{i=1}^{j-1} r_i\right)^{a+b} \cdot B(a,b+1)
\end{align*}
The last transformation uses the fact that $B(x,y) = \int_0^1 t^{x-1} (1-t)^{y-1} dt$.
\end{proof}

\begin{lemma}
\label{lemma:RadiusIntegral}
For any natural number $k > 0$, any $r_1,\dots,r_k,n_1,\dots,n_k > 0$, $n = \sum_{i=0}^k n_i$, $r = \sum_{i=0}^k r_i$ the following equation holds:
$$\int_0^r r_1^{n_1 - 1} \int_0^{r-r_1} r_2^{n_2 - 1} \dots \int_0^{r-\sum_{i=1}^{k-1} r_i} r_k^{n_k - 1} dr_k \dots dr_1 = \frac{r^n}{\Gamma(n+1)} \prod_{i=1}^k \Gamma(n_i)$$
\end{lemma}
\begin{proof}

Using Lemma \ref{lemma:IntegralBetaFunction}, we can solve the innermost integral by setting $j=k, a= n_k, b=0$, which gives us $B(n_k,1) \cdot \left(r- \sum_{i=1}^{k-1} r_i\right)^{n_k}$. Therefore:
\begin{align*}
I &= \int_0^r r_1^{n_1 - 1} \int_0^{r-r_1} r_2^{n_2 - 1} \dots \int_0^{r-\sum_{i=1}^{k-1} r_i} r_k^{n_k - 1} dr_k \dots dr_1\\
&= B(n_k,1) \cdot \int_0^r r_1^{n_1 - 1} \dots \int_0^{r-\sum_{i=1}^{k-2} r_i} r_{k-1}^{n_{k-1} - 1} \cdot \left(r- \sum_{i=1}^{k-1} r_i\right)^{n_k} dr_{k-1} \dots dr_1
\end{align*}
As one can see, we can again apply Lemma \ref{lemma:IntegralBetaFunction} to the innermost integral. Repeatedly applying Lemma \ref{lemma:IntegralBetaFunction} finally results in:
$$I = B(n_k,1)\cdot B(n_{k-1},n_k + 1) \cdot \dots \cdot B(n_1, n_2+\dots+n_k+1)\cdot r^{n_1+\dots+n_k}$$
We use that $B(x,y) = \frac{\Gamma(x)\Gamma(y)}{\Gamma(x+y)}$ in order to rewrite this equation:
$$I = r^{n_1+\dots+n_k} \cdot \frac{\Gamma(n_k)\Gamma(1)}{\Gamma(n_k+1)} \cdot \frac{\Gamma(n_{k-1})\Gamma(n_k+1)}{\Gamma(n_{k-1}+n_k+1)} \cdots \frac{\Gamma(n_1)\Gamma(n_2+\dots+n_k+1)}{\Gamma(n_1+n_2+\dots+n_k+1)}$$
Because $\Gamma(1) = 1$ and because most of the terms cancel out, this reduces to:
$$I = r^{n_1+\dots+n_k} \cdot \Gamma(n_k) \cdots \Gamma(n_1) \cdot \frac{1}{\Gamma(n_1+\dots+n_k+1)} = \frac{r^n}{\Gamma(n+1)} \prod_{i=1}^{k} \Gamma(n_i)$$
\end{proof}
Using these three lemata, we can now derive the size of a hyperball in a conceptual space without domain and dimension weights: 

\begin{proposition}
\label{proposition:HyperballVolume}
The hypervolume of a hyperball with radius $r$ in a space with the unweighted combined metric $d_C^{\Delta}$ and the domain structure $\Delta$ can be computed in the following way, where $n$ is the overall number of dimensions and $n_\delta$ is the number of dimensions in domain $\delta$:
$$V(r,\Delta) = \frac{r^n}{n!} \prod_{\delta \in \Delta} \left(n_\delta! \frac{\pi^{\frac{n_\delta}{2}}}{\Gamma\left(\frac{n_\delta}{2}+1\right)}\right)$$
\end{proposition}
\begin{proof}

The hyperball can be defined as the set of  all points that have a distance of maximally $r$ to the origin, i.e., $$H = \Big \{x \in CS \;|\;d_C^\Delta(x,0) = \sum_{\delta \in \Delta} \sqrt{\sum_{d \in \delta} x_{d}^2} \leq r\Big \}$$
If we define $\forall \delta \in \Delta: r_\delta := \sqrt{\sum_{d \in \delta} x_{d}^2}$, we can easily see that $\sum_{\delta \in \Delta} r_\delta \leq r$. The term $r_\delta$ can be interpreted as the distance between $x$ and the origin within the domain $\delta$. The constraint $\sum_{\delta \in \Delta} r_\delta \leq r$ then simply means that the sum of domain-wise distances is less than the given radius. One can thus interpret $r_\delta$ as the radius within domain $\delta$.

We would ultimately like to compute
$$V(r,\Delta) = \int \dots \int_H 1\; dH$$ 
This integration becomes much easier if we use spherical coordinates instead of the Cartesian coordinates provided by our conceptual space.\\

Let us first consider the case of a single domain $\delta$ of size $n$. A single domain corresponds to a standard Euclidean space, therefore we can use the standard procedure of changing to spherical coordinates (cf., e.g., \cite{DeRise1992}). Let us index the dimensions of $\delta$ as $d_1,\dots,d_n$. The coordinate change within the domain $\delta$ then looks like this:
\begin{align*}
x_{1} &= t \cdot \cos(\phi_{1})\\
x_{2} &= t \cdot \sin(\phi_{1}) \cdot \cos(\phi_{2})\\
&\hspace{0.2cm}\vdots\\
x_{n-1} &= t \cdot \sin(\phi_{1}) \cdots \sin(\phi_{n-2}) \cdot \cos(\phi_{n-1})\\
x_{n} &= t \cdot \sin(\phi_{1}) \cdots \sin(\phi_{n-2}) \cdot \sin(\phi_{n-1})
\end{align*}

In order to switch the integral to spherical coordinates, we need to calucalate the volume element. This can be found by looking at the determinant of the transformation's Jacobian matrix. The Jacobian matrix of the transformation of a single domain $\delta$ can be written as follows:

\begin{gather*}
 J_{\delta} =
  \left[ {\begin{array}{cccc}
   \frac{\delta x_{1}}{\delta t} & \frac{\delta x_{1}}{\delta \phi_{1}} & \dots & \frac{\delta x_{1}}{\delta \phi_{n-1}} \\
   \frac{\delta x_{2}}{\delta t} & \frac{\delta x_{2}}{\delta \phi_{1}} & \dots & \frac{\delta x_{2}}{\delta \phi_{n-1}} \\
   \vdots & \vdots & \ddots & \vdots \\
   \frac{\delta x_{n}}{\delta t} & \frac{\delta x_{n}}{\delta \phi_{1}} & \dots & \frac{\delta x_{n}}{\delta \phi_{n-1}} \\
  \end{array} } \right] =\\[1ex]
  \hspace{-1cm}\left[ {\tiny\begin{array}{cccccc}
	\cos(\phi_{1}) 							&-t\sin(\phi_{1}) 						&0 &0 &\dots &0 \\
	\sin(\phi_{1})\cos(\phi_{2})		&t\cos(\phi_{1})\cos(\phi_{2})	&-t\sin(\phi_{1})\sin(\phi_{2}) &0   &\dots &0\\
	\vdots &\vdots &\vdots &\vdots &\vdots &\vdots\\
	\sin(\phi_{1})\cdots\sin(\phi_{n-2})\cos(\phi_{n-1}) &\dots &\dots &\dots &\dots &-t\sin(\phi_{1})\cdots\sin(\phi_{n-2})\sin(\phi_{n-1})\\
	\sin(\phi_{1})\cdots\sin(\phi_{n-2})\sin(\phi_{n-1}) &\dots &\dots &\dots &\dots &t\sin(\phi_{1})\cdots\sin(\phi_{n-2})\cos(\phi_{n-1})
  \end{array} } \right]
\end{gather*}
The determinant of this matrix can be computed like this:
$$det(J_{\delta}) = t^{n-1}\cdot\sin^{n-2}(\phi_1)\cdot\sin^{n-3}(\phi_2)\cdots\sin(\phi_{n-2})$$
We can now perform the overall switch from Cartesian to spherical coordinates by performing this coordinate change for each domain individually. Let us index the Cartesian coordinates of a point $x$ in domain $\delta$ by $x_{\delta,1},\dots,x_{\delta,n_\delta}$. Let us further index the spherical coordinates of domain $\delta$ by $r_\delta$ and $\phi_{\delta,1},\dots,\phi_{\delta,n_\delta-1}$. Let $k = |\Delta|$ denote the total number of domains.

Because $x_{\delta,j}$ is defined independently from $r_{\delta'}$ and $\phi_{\delta',j'}$ for different domains $\delta \neq \delta'$, any derivative $\frac{x_{\delta,j}}{r_{\delta'}}$ or $\frac{x_{\delta,j}}{\phi_{\delta',j'}}$ will be zero. If we apply the coordinate change to all domains at once, the Jacobian matrix of the overall transformation has therefore the structure of a block matrix:

$$
   J=
  \left[ {\begin{array}{cccc}
   J_1 & 0 & \dots & 0 \\
   0 & J_2 & \dots & 0 \\
   \vdots & \vdots & \ddots & \vdots \\
   0 & 0 & \dots & J_k \\
  \end{array} } \right]
$$

The blocks on the diagonal are the Jacobian matrices of the individual domains as defined above, and all other blocks are filled with zeroes because all cross-domain derivatives are zero.
Because the overall $J$ is a block matrix, we get that $det(J) = \prod_{\delta \in \Delta} det(J_\delta)$ (cf. \cite{Silvester2000}).
Our overall volume element is thus
$$det(J) = \prod_{\delta \in \Delta} det(J_\delta) = \prod_{\delta \in \Delta} r_\delta^{n_\delta-1}\sin^{n_\delta-2}(\phi_{\delta,1})\sin^{n_\delta-3}(\phi_{\delta,2})\cdots \sin(\phi_{\delta,n_\delta-2})$$

The limits of the angle integrals are $[0,2\pi]$ for the outermost and $[0,\pi]$ for all other integrals.
Based on our constraint $\sum_{\delta \in \Delta} r_\delta \leq r$, we can derive the limits for the integrals over the $r_\delta$ as follows, assuming an arbitrarily ordered indexing $\delta_1,\dots,\delta_k$ of the domains:
\begin{align*}
r_1 &\in [0,r]\\
r_2 &\in [0,r - r_1]\\
r_3 &\in [0,r - r_1 - r_2]\\
&\hspace{0.2cm}\vdots\\
r_k &\in [0,r - \sum_{i=1}^{k-1} r_i]
\end{align*}
The overall coordinate change therefore looks like this:
\begin{align*}
V(r,\Delta) &= \int \dots \int_H 1\; dH \\
&=
\underbrace{\int\displaylimits_{\phi_{1,n_1-1}=0}^{2\pi} \int\displaylimits_{\phi_{1,n_1-2}=0}^\pi
\cdots \int\displaylimits_{\phi_{1,1}=0}^\pi \int\displaylimits_{r_1 = 0}^r}_{\delta = 1} 
\cdots
\underbrace{\int\displaylimits_{\phi_{k,n_k-1}=0}^{2\pi} \int\displaylimits_{\phi_{k,n_k-2}=0}^\pi
\cdots \int\displaylimits_{\phi_{k,1}=0}^\pi \int\displaylimits_{r_k = 0}^{r - \sum_{i=1}^{k-1} r_i}}_{\delta = k}\\[2ex]
&\hspace{1cm}\underbrace{r_1^{n_1-1}\sin^{n_1-2}(\phi_{1,1})\cdots \sin(\phi_{1,n_1-2})}_{\delta = 1} 
\cdots 
\underbrace{r_k^{n_k-1}\sin^{n_k-2}(\phi_{k,1})\cdots \sin(\phi_{k,n_k-2})}_{\delta = k}\\[2ex]
&\hspace{1cm}\underbrace{dr_k d\phi_{k,1}\dots d\phi_{k, n_k - 1}}_{\delta = k}
\dots
\underbrace{dr_1 d\phi_{1,1}\dots d\phi_{1, n_1 - 1}}_{\delta = 1}
\end{align*}
\begin{align*}
&=
\underbrace{\int\displaylimits_{0}^{2\pi} \int\displaylimits_{0}^\pi \cdots \int\displaylimits_{0}^\pi \sin^{n_1-2}(\phi_{1,1})\cdots \sin(\phi_{1,n_1-2}) \;d\phi_{1,1}\dots d\phi_{1, n_1 - 1}}_{\delta = 1} \\[2ex]
&\hspace{1cm}\cdots\quad
\underbrace{\int\displaylimits_{0}^{2\pi} \int\displaylimits_{0}^\pi
\cdots \int\displaylimits_{0}^\pi \sin^{n_k-2}(\phi_{k,1})\cdots \sin(\phi_{k,n_k-2})\;d\phi_{k,1}\dots d\phi_{k, n_k - 1}}_{\delta = k}\\[2ex]
&\hspace{1cm}\int\displaylimits_{0}^r r_1^{n_1-1}\cdots\int\displaylimits_{0}^{r - \sum_{i=1}^{k-1} r_i} r_k^{n_k-1} \; dr_1\dots dr_k
\end{align*}
By applying Lemma \ref{lemma:AngleIntegral} and Lemma \ref{lemma:RadiusIntegral}, we can write this as:
\begin{align*}
V(r,\Delta) &= \left(2 \cdot \frac{\pi^{\frac{n_1}{2}}}{\Gamma(\frac{n_1}{2})}\right) \cdots \left(2 \cdot \frac{\pi^{\frac{n_k}{2}}}{\Gamma(\frac{n_k}{2})}\right) \cdot \frac{r^n}{\Gamma(n+1)} \prod_{i=1}^{k} \Gamma(n_i)\\
&= \frac{r^n}{\Gamma(n+1)} \cdot \prod_{i=1}^{k} \left(2 \cdot \pi^{\frac{n_i}{2}} \cdot \frac{\Gamma(n_i)}{\Gamma(\frac{n_i}{2})}\right)
\end{align*}
We can simplify this formula further by using the identity $\forall n \in \mathbb{N}: \Gamma(n+1) = n!$ and the rewrite $\prod_{i=0}^k \widehat{=} \prod_{\delta \in \Delta}$:

\begin{align*}
V(r,\Delta) &= \frac{r^n}{n!} \cdot \prod_{\delta \in \Delta} \left(2 \cdot \pi^{\frac{n_\delta}{2}} \cdot \frac{(n_\delta - 1)!}{\Gamma(\frac{n_\delta}{2})}\right) 
= \frac{r^n}{n!} \cdot \prod_{\delta \in \Delta} \left(\frac{2}{n_\delta} \cdot n_\delta! \cdot \frac{\pi^{\frac{n_\delta}{2}}}{\Gamma(\frac{n_\delta}{2})}\right)\\
&= \frac{r^n}{n!} \cdot \prod_{\delta \in \Delta} \left(n_\delta! \cdot \frac{\pi^{\frac{n_\delta}{2}}}{\frac{n_\delta}{2} \cdot \Gamma(\frac{n_\delta}{2})}\right)
= \frac{r^n}{n!} \cdot \prod_{\delta \in \Delta} \left(n_\delta! \cdot \frac{\pi^{\frac{n_\delta}{2}}}{\Gamma(\frac{n_\delta}{2}+1)}\right)
\end{align*}
The last transformation uses the fact that $\forall x \in \mathbb{R}^+: \Gamma(x) \cdot x = \Gamma(x+1)$.
\end{proof}

\section{Hyperballs under the Weighted Metric}
\label{Hyperellipses}

We now generalize our results from the previous section from the unweighted to the weighted combined metric $d_C^ \Delta$.

\begin{proposition}
The hypervolume of a hyperball with radius $r$ in a space with the weighted combined metric $d_C^\Delta$, the domain structure $\Delta$, and the set of weights $W$ can be computed by the following formula, where $n$ is the overall number of dimensions and $n_\delta$ is the number of dimension in domain $\delta$:
$$V(r,\Delta, W) = \frac{1}{\prod_{\delta \in \Delta} w_{\delta} \cdot \prod_{d \in \delta} \sqrt{w_d}} \cdot \frac{r^n}{n!} \cdot \prod_{\delta \in \Delta} \left(n_\delta! \cdot \frac{\pi^{\frac{n_\delta}{2}}}{\Gamma(\frac{n_\delta}{2}+1)}\right)$$
\end{proposition}
\begin{proof}
As G\"ardenfors has already argued in \cite{Gardenfors2000}, putting weights on dimensions in a conceptual space is equivalent to stretching each dimension of the unweighted space by the weight assigned to it.

If the overall radius of a ball is $r$, and some dimension has the weight $w$, then the farthest away any point $x$ can be from the origin on this dimension must satisfy $w\cdot x = r$, i.e., $x = \frac{r}{w}$. That is, the ball needs to be stretched by a factor $\frac{1}{w}$ in the given dimension, thus its hypervolume also changes by a factor of $\frac{1}{w}$. A hyperball under the weighted metric is thus equivalent to a hyperellipse under the unweighted metric.

In our case, the weight for any dimension $d$ within a domain $\delta$ corresponds to $w_{\delta} \cdot \sqrt{w_{d}}$: If we look at a point $x$ with coordinates $(0, \dots, 0, x_d, 0, \dots, 0)$, then $d(0,x) = w_\delta \cdot \sqrt{w_d \cdot x_d^2} = w_\delta \cdot \sqrt{w_d} \cdot x_d$ (with $\delta$ being the domain to which the dimension $d$ belongs). If we multiply the size of the hyperball by $\frac{1}{w_\delta \cdot \sqrt{w_d}}$ for each dimension $d$, we get:
\begin{align*}
V(r,\Delta,W) &= \frac{1}{\prod_{\delta \in \Delta}\prod_{d \in \delta} w_{\delta} \sqrt{w_{d}}} \cdot V(r, \Delta)\\
&=\frac{1}{\prod_{\delta \in \Delta}\prod_{d \in \delta} w_{\delta} \sqrt{w_{d}}} \cdot \frac{r^n}{n!} \cdot \prod_{\delta \in \Delta} \left(n_\delta! \cdot \frac{\pi^{\frac{n_\delta}{2}}}{\Gamma(\frac{n_\delta}{2}+1)}\right)
\end{align*}
This is the hypervolume of a hyperball under the weighted combined metric. 
\end{proof}

\bibliographystyle{plain}
\bibliography{/home/lbechberger/Documents/Papers/jabref.bib}

\end{document}